\documentclass[10pt, letterpaper, twocolumn]{IEEEtran}
\IEEEoverridecommandlockouts
\usepackage{setspace}
\ifCLASSINFOpdf
\else
   \usepackage[dvips]{graphicx}
\fi
\usepackage{url}
\usepackage{amsmath}
\usepackage{amsthm}
\usepackage{amssymb}
\usepackage[ruled,vlined]{algorithm2e}
\newtheorem{theorem}{Theorem}

\def\BibTeX{{\rm B\kern-.05em{\sc i\kern-.025em b}\kern-.08em T\kern-.1667em\lower.7ex\hbox{E}\kern-.125emX}}

\usepackage{graphicx}

\DeclareMathOperator*{\argmin}{arg\,min}

\usepackage[nospace]{cite}
\usepackage{hyperref}
\usepackage{caption}
\usepackage{subcaption}
\title{Decentralized Smoothing ADMM for Quantile Regression with Non-Convex Sparse Penalties}
\author{Reza Mirzaeifard~\IEEEmembership{Student Member,~IEEE}, 
        Diyako Ghaderyan, 
        Stefan Werner~\IEEEmembership{Fellow,~IEEE}% Avoid using 
\thanks{Stefan Werner and Reza Mirzaeifard are with the Department
of Electronic Systems, Norwegian University of Science and Technology-NTNU, Norway,
Trondheim, 7032 Norway (e-mail: \{stefan.werner, reza.mirzaeifard\}@ntnu.no). Stefan Werner is also with the Department of Information and Communications Engineering, Aalto University, 00076, Finland. %}
Diyako Ghaderyan is with the Department of Information and Communications Engineering, Aalto University, 00076, Finland (e-mail: diyako.ghaderyan@aalto.fi).%}
This work was partially supported by the Research Council of Norway and the Research Council of Finland (Grant 354523).}
}
\begin{document}
\maketitle
%\IEEEpubidadjcol
\begin{abstract}
In the rapidly evolving internet-of-things (IoT) ecosystem, effective data analysis techniques are crucial for handling distributed data generated by sensors. Addressing the limitations of existing methods, such as the sub-gradient approach, which fails to distinguish between active and non-active coefficients effectively, this paper introduces the decentralized smoothing alternating direction method of multipliers (DSAD) for penalized quantile regression. Our method leverages non-convex sparse penalties like the minimax concave penalty (MCP) and smoothly clipped absolute deviation (SCAD), improving the identification and retention of significant predictors. DSAD incorporates a total variation norm within a smoothing ADMM framework, achieving consensus among distributed nodes and ensuring uniform model performance across disparate data sources. This approach overcomes traditional convergence challenges associated with non-convex penalties in decentralized settings. We present convergence proof and extensive simulation results to validate the effectiveness of the DSAD, demonstrating its superiority in achieving reliable convergence and enhancing estimation accuracy compared with prior methods.
\end{abstract}
\IEEEpeerreviewmaketitle
\section{Introduction}
\label{sec:intro}
The expansion of cyber-physical systems in the internet-of-things (IoT) ecosystem has prompted a shift from centralized to decentralized data processing, primarily due to the immense computational demands, bandwidth constraints, and privacy issues associated with central data aggregation \cite{chen2019quantile,zhou2022admm}. In this context, decentralized learning has emerged as a pivotal methodology, allowing collaborative model training across dispersed devices while maintaining data locality, thereby addressing significant privacy and bandwidth concerns \cite{yin2021comprehensive,zhou2023decentralized}.

A persistent challenge in decentralized learning is the presence of outliers, particularly in heavy-tailed data distributions. These outliers can severely skew model accuracy and reliability \cite{zhao2022participant,tsouvalas2022federated}. Traditional linear regression methods, which estimate the conditional mean, often fall short in such conditions due to their sensitivity to outliers. In contrast, quantile regression, which estimates conditional quantiles rather than means, provides a more robust framework in these scenarios, offering valuable insights into the underlying data distribution across various quantiles \cite{yu2020probabilistic,taieb2016fore}.

To enhance the sparsity and interpretability of models, particularly in high-dimensional data scenarios, non-convex penalties like the minimax concave penalty (MCP) \cite{fan2001variable} and smoothly clipped absolute deviation (SCAD) \cite{zhang2010nearly} have been recognized for their ability to reduce bias more effectively than $l_1$ penalties. These non-convex penalties improve model accuracy by selectively shrinking coefficients, thereby identifying and zeroing out non-significant predictors, a critical feature missing in many sparse modeling techniques \cite{mirzaeifard2022dynamic,mirzaeifard2022admm}.

However, current decentralized quantile regression methods using sub-gradient approaches \cite{mirzaeifard2023distributed,wang2017distributed}  often fail to effectively distinguish between active and non-active coefficients, resulting in models that are either overly dense or insufficiently precise. This gap highlights the need for a new approach capable of handling non-convex optimization complexities in a distributed setting without compromising sparsity or accuracy.

Our contribution is the decentralized smoothing ADMM (DSAD), an innovative approach designed to address these challenges within the decentralized learning framework. DSAD leverages the total variation $l_1$ norm and utilizes smoothing techniques that transform the non-convex problem into a tractable smooth approximation, enabling efficient convergence to a local minimum and consensus across nodes.

This letter's key contributions are:
\begin{itemize}
\item The introduction of DSAD, a decentralized algorithm specifically tailored for handling non-convex, non-smooth penalties in quantile regression. 
\item The demonstration of DSAD's capability to effectively manage the sparsity and accuracy challenges posed by non-convex penalties, achieving node consensus and local optimality in a decentralized setting.
\item Empirical validation of DSAD through simulations, showcasing superior performance and setting a new benchmark in distributed penalized quantile regression.
\end{itemize}
\noindent\textit{\textbf{Mathematical Notations}}: Lowercase letters represent scalars, bold lowercase letters represent column vectors, and bold uppercase letters represent matrices. The transpose of a matrix is represented by $(\cdot)^\text{T}$, and the $j$th column of matrix $\mathbf{A}$ is denoted by $\mathbf{a}_{j}$. Additionally, the entry in the $i$th row and $j$th column of $\mathbf{A}$ is denoted by $a_{ij}$. Moreover, we let $\mathbf{A}_{:,<s} \mathbf{x}_{<s} := \sum_{i<s} \mathbf{A}_{:,i} x_i$ and, similarly, $\mathbf{A}_{:,>s} \mathbf{x}_{>s} := \sum_{i>s} \mathbf{A}_{:,i} x_i$. Finally, $\partial f(u)$ represents the sub-gradient of function $f(\cdot)$ evaluated at $u$. 
\section{Problem formulation}
Consider a scalar response variable \(Y\) and a \(P\)-dimensional vector of predictors \(\mathbf{x}\). The conditional cumulative distribution function is defined as \(F_{Y}(y|\mathbf{x}) = P(Y \leq y | \mathbf{x})\), and the \(\tau\)th conditional quantile, for \(\tau \in (0,1)\), is given by $Q_Y(\tau|\mathbf{x}) = \inf\{y : F_Y(y|\mathbf{x}) \geq \tau\}$.
In quantile regression, this quantile can be linearly modeled as:
\begin{equation}
    Q_Y(\tau|\mathbf{x}) = \mathbf{x}^\text{T} \boldsymbol{\beta}_{\tau} + q_{\tau}^{\epsilon}
\end{equation}
where \(\boldsymbol{\beta}_{\tau}\) are the regression coefficients specific to quantile \(\tau\), and \(q_{\tau}^{\epsilon}\) represents the \(\tau\)th quantile of the noise \cite{koenker1982robust}.

To estimate the model parameters from a dataset \(\{\mathbf{x}_i, y_i\}_{i=1}^{n}\), where $n$ is the number of samples, and leveraging \emph{a priori} information about the model coefficients by incorporating the penalty function $P_{\lambda,\gamma}\mathopen{}\left(\mathbf{w}\right)\mathclose{}$, one can solve:
\begin{equation}
    \hat{\mathbf{w}} = \argmin_{\mathbf{w}} \frac{1}{n} \sum_{i=1}^{n} \rho_{\tau}(y_i - \mathbf{\bar{x}}_i^\text{T} \mathbf{w})+P_{\lambda,\gamma}\mathopen{}\left(\mathbf{w}\right)
\end{equation}
where \( \mathbf{\bar{x}}_i = [\mathbf{x}_i^\text{T}, 1]^\text{T} \in\mathbb{R}^{P+1} \), \( \mathbf{w} = [\boldsymbol{\beta}_{\tau}^\text{T}, q_{\tau}^{\epsilon}]^\text{T} \in\mathbb{R}^{P+1}\), and \( \rho_{\tau}(u) = \frac{1}{2}\mathopen{}\left(|u| + \mathopen{}\left( 2\tau-1\right)\mathclose{}u\right)\mathclose{} \) is the check loss function \cite{koenker1982robust}. 

To enhance sparsity more effectively than the widely used \(l_1\) norm, which can introduce estimation bias by overly shrinking large coefficients, this paper employs weakly convex penalties such as MCP \cite{fan2001variable} and SCAD \cite{zhang2010nearly}. Defined as \(P_{\lambda,\gamma}(\mathbf{w}) = \sum_{p=1}^P g_{\lambda,\gamma}(w_p)\), these penalties are specifically chosen for their ability to reduce bias and improve the selection of significant predictors while maintaining weak convexity, thus providing a superior approach to regularization.
% \subsection{Distributed Sparse Quantile Regression}

In a distributed setting, we model a network with $L$ agents as an undirected graph $\mathcal{G}$ with vertices $\mathcal{V} = \{1,\cdots,L\}$ connected by bidirectional edges $\mathcal{E}$. Each agent $l \in \mathcal{V}$ communicates with its neighbors in $\mathcal{N}_l$, where $|\mathcal{N}_l|$ denotes the number of neighbors. The observation matrix at agent $l$ is denoted by $\mathbf{X}^{(l)} = [\bar{\mathbf{x}}^{(l)}_{1},\cdots,\bar{\mathbf{x}}^{(l)}_{M_l}]^\text{T} \in \mathbb{R}^{M_l \times (P+1)}$, and the corresponding response vector is $\mathbf{y}^{(l)} = [{y}^{(l)}_{1},\cdots,{y}^{(l)}_{M_l}]^\text{T} \in \mathbb{R}^{M_l}$. The cumulative number of measurements across all agents is $\sum_{l=1}^L M_l = n$.

Distributed penalized quantile regression aims to estimate $\mathbf{w}_l$ for each agent by solving the following formulation:
\begin{alignat}{2}
\nonumber &\min_{\{\mathbf{w}_l, \mathbf{z}_l\}_{l=1}^{L}} &\quad& \sum_{l=1}^{L} \left(\frac{1}{2} \left(\|\mathbf{z}_l\|_1 + (2\tau-1) \mathbf{1}_n^\text{T} \mathbf{z}_l\right) + n P_{\lambda,\gamma}(\mathbf{w}_l)\right)
\\ 
&\text{subject to} & & \mathbf{z}_l + \mathbf{X}^{(l)}\mathbf{w}_l = \mathbf{y}^{(l)}, \forall l \in \{1,\cdots,L\}
\\
& & & \mathbf{w}_l = \mathbf{w}_j \hspace{2mm} \forall j \in \mathcal{N}_l
\end{alignat}
Despite the potential of distributed ADMM for solving non-convex and non-smooth optimization problems, current implementations do not ensure convergence in a single-loop algorithm, which is a limitation especially when traditional methods consider smooth or convex formulations \cite{wang2019global, hong2015convergence, yashtini2020convergence, themelis2020douglas,mirzaeifard2022robust}.

The application of the total variation norm has been proposed to enforce consensus more effectively. By reformulating the problem, we incorporate a consensus regularization term:
\begin{alignat}{2}\label{eq10} \nonumber
&\min_{\{\mathbf{w}_l,\mathbf{z}_l\}_{l=1}^{L}} &\qquad& \sum_{l=1}^{L} \frac{1}{2}\mathopen{}\left(\|\mathbf{z}_l\|_1 + \mathopen{}\left( 2\tau-1\right)\mathclose{}\mathbf{1}_n^{\text{T}} \mathbf{z}_l\right)\mathclose{}\\ \nonumber
&\qquad & &+ \hspace{0.2mm}n \hspace{1mm} P_{\lambda,\gamma}\mathopen{}\left(\mathbf{w}_l\right)\mathclose{} +\omega\sum_{j\in \mathcal{N}_l,j>l}\|\mathbf{g}_{lj}-\mathbf{g}_{jl}\|_1
\\ \nonumber
&\text{subject to} & & \mathbf{z}_l+\mathbf{X}^{(l)}\mathbf{w}_l= \mathbf{y}^{(l)}, \forall i \in \{1,\cdots,L\}
\\ 
& \quad & & \mathbf{g}_{lj}=\mathbf{w}_l, \quad \mathbf{g}_{jl}=\mathbf{w}_j, \forall j \in \mathcal{N}_l, j>l
\end{alignat}
where $\omega$ is the weight of the total variation norm and $\mathbf{G} = \{\{\mathbf{g}_{lj}\}_{j \in \mathcal{N}_l, j > l}\}_{l=1}^{L}$ are auxiliary variables designed to aid in achieving a consensus. By following a smoothing approach similar to that in \cite{mirzaeifard2023smoothing}, it is possible to show that the algorithm converges to a stationary point. The challenge remains to demonstrate that each stationary point reached by the ADMM satisfies the consensus condition and is optimal.
\section{Decentralized Smoothing ADMM}
% To enhance the optimization of non-smooth functions within ADMM frameworks, smoothing techniques are applied. These techniques convert non-smooth functions \( g \) into smooth functions \( \tilde{g} \), facilitating algorithmic processes. A smoothing function \( \tilde{g} \), suitable for locally Lipschitz functions \( g \), ensures better manageability of the optimization process through its properties: continuity, differentiability, and convergence \cite{chen2012smoothing}. \( \tilde{g}(\mathbf{x}, \mu) \) converges to \( g(\mathbf{x}) \) as \( \mu \to 0^+ \), aligning closely with the original function, supported by beneficial gradient properties \cite{chen2012smoothing}.
To enhance ADMM framework optimization of non-smooth functions, smoothing techniques transform a function  \( g \) into a series of smooth functions \( \tilde{g} \), suitable for locally Lipschitz environments, facilitating the optimization process via properties like continuity, differentiability, and gradient properties, ensuring \( \tilde{g}(\mathbf{x}, \mu) \) converges to \( g(\mathbf{x}) \) as \(\mu \to 0^+\) \cite{chen2012smoothing}.

Our smoothing ADMM algorithm employs these properties to approximate the $\|\mathbf{z}\|_1$ norm effectively %$\|\mathbf{z}\|_1$ norm 
by using a sum of smooth functions for each $|z_i|$ presented in  \cite{chen2012smoothing} as $h\mathopen{}\left(\mathbf{z},\mu\right)\mathclose{}=\sum_{i=1}^{n} f\mathopen{}\left(z_i,\mu\right)\mathclose{}$, where: 
\begin{equation}\label{eq12}
 f\mathopen{}\left(z_i,\mu\right)\mathclose{}=
\begin{cases}
|z_i|, &  \mu \leq |z_i|
\\
\frac{z_i^2}{2\mu}+\frac{\mu}{2}. &  |z_i| < \mu 
\end{cases}
\end{equation}
Using approximation \eqref{eq12}, the following approximate augmented Lagrangian can be derived:
\begin{multline}\label{eq13}
{\bar{\mathcal{L}}}_{\sigma_{\Psi},\sigma_{\xi},\mu}\mathopen{}\left(\mathbf{W},\mathbf{Z},\mathbf{G},\boldsymbol{\Psi},\boldsymbol{\xi}\right)\mathclose{} =  \sum_{l=1}^{L} \Bigg( \frac{1}{2}\mathopen{}h\left(\mathbf{z}_l,\mu\right)+ \tau\mathbf{1}_{M_l}^{\text{T}} \mathbf{z}_l- \\ \frac{1}{2}\mathbf{1}_{M_l}^{\text{T}} \mathbf{z}_l + M_l P_{\lambda,\gamma}\mathopen{}\left(\mathbf{w}_l\right)\mathclose{}  
+  \frac{\sigma_{\Psi}}{2}\mathopen{}\left\|\mathbf{z}_l+ \mathbf{X}^{(l)}\mathbf{w}_l- \mathbf{y}^{(l)}+\frac{\boldsymbol{\Psi}_l}{\sigma_{\Psi}}\right\|_2^2 \mathclose{} +\\
      \sum_{j\in \mathcal{N}_l,j>l}\bigg(\omega \hspace{0.5mm} h\left(\mathbf{g}_{lj}-\mathbf{g}_{jl},\mu\right)+\frac{\sigma_{\xi}}{2} \|\mathbf{w}_{i}-\mathbf{g}_{lj}\|^2_2+\frac{\sigma_{\xi}}{2}\|\mathbf{w}_{j}-\mathbf{g}_{jl}\|^2_2 \\ 
+\boldsymbol{\xi}_{lj}^{\text{T}} (\mathbf{w}_{l}-\mathbf{g}_{lj}) + \boldsymbol{\xi}_{jl}^{\text{T}} (\mathbf{w}_{j}-\mathbf{g}_{jl})\bigg)\Bigg)
\end{multline}
where $\mathbf{W}=[\mathbf{w}_1,\cdots,\mathbf{w}_L]$, $\mathbf{Z}=[\mathbf{z}_1,\cdots,\mathbf{z}_L]$, $\boldsymbol{\xi}= \{\{\boldsymbol{\xi}_{lj}\}_{j \in \mathcal{N}_l, j > l}\}_{l=1}^{L}$ and $\boldsymbol{\Psi}= \{\boldsymbol{\Psi}_{l}\}_{l=1}^{L}$ are dual variables and $\sigma_{\Psi}$ and $\sigma_{\xi}$ are penalty parameters. 

To iteratively adjust the approximation, we update the parameters $\mu$, $\sigma_{\Psi}$, $\sigma_{\xi}$ in each iteration using the following rules, where $c>0$, $d>0$ and $\beta>0$:
\begin{align}\label{eq:up:sm}
\sigma_{\Psi}^{(k+1)} = c \sqrt{k+1}, \
\sigma_{\xi}^{(k+1)} = d \sqrt{k+1}, \
\mu^{(k+1)} = \frac{\beta}{\sqrt{k+1}}.
\end{align}

The update process for each $\mathbf{w}_l$ is divided into $P+1$ sequential steps, with the $p$th element updated in its corresponding step. After simplifications,  the update formula, at each agent $l$, for the $p$th for $p\leq P$ is as follows:
\begin{align}\label{up:wp}
w_{p,l}^{\mathopen{}\left(k+1\right)\mathclose{}}= \textbf{Prox}_{ g_{\lambda,\gamma}}\mathopen{}\left(\frac{a_{p,l}}{\Upsilon^{(l)}_p};\frac{M_l}{\Upsilon^{(l)}_p}\right)\mathclose{}
\end{align}
where $\textbf{Prox}_{f}\mathopen{}\left(w;\gamma\right)\mathclose{}= \argmin_x \mathopen{}\left \{f\mathopen{}\left(x\right)\mathclose{}+\frac{1}{2\gamma}\mathopen{}\left\|x-w\right\|_2^2\mathclose{}\right \}$,
\begin{multline}\label{up:wp:a}
    a_{p,l} = -\left(\mathbf{x}^{(l)}_{:,p}\right)^\text{T} \left(\mathbf{X}^{(l)}_{:,<p}\mathbf{w}^{\mathopen{}\left(k+1\right)\mathclose{}}_{<p,l} + \mathbf{X}^{(l)}_{:,>p}\mathbf{w}^{\mathopen{}\left(k\right)\mathclose{}}_{>p,l}\right) \\
    -\left(\boldsymbol{\Psi}_{l}^{\mathopen{}\left(k\right)\mathclose{}}\right)^\text{T} \mathbf{X}^{(l)}_{:,p}
    + \sigma_{\Psi}^{\mathopen{}\left(k+1\right)\mathclose{}} \left(\mathbf{y}^{(l)} - \mathbf{z}_l^{\mathopen{}\left(k\right)\mathclose{}}\right)^\text{T} \mathbf{X}^{(l)}_{:,p} +\\
     \sum_{j \in \mathcal{N}_l, j < l} (\sigma_{\xi}^{(k+1)} g_{lj,p}^{(k)} - \xi_{lj,p}^{(k)})
    + \sum_{j \in \mathcal{N}_l, j > l} (\sigma_{\xi}^{(k+1)} g_{jl,p}^{(k)} - \xi_{jl,p}^{(k)})
\end{multline} 
and
\begin{equation}
\Upsilon^{(l)}_p=\sigma_{\Psi}^{\mathopen{}\left(k+1\right)\mathclose{}}\mathopen{}\left\|\mathbf{X}^{(l)}_{:,p}\right\|_2^2\mathclose{}+\kappa_{\xi}^{(k+1)}
\end{equation}
Both MCP and SCAD admit closed-form solutions of the proximal operator \cite{huang2012selective}. For $(P+1)$th element as the penalty function is zero we have: 
\begin{equation}\label{eq:up:w:p+1}
w_{P+1,l}=\frac{a_{P+1,l}}{\Upsilon^{(l)}_{P+1}}
\end{equation}

Next, the update of each $\mathbf{z}_l$ can be formulated as:
\begin{align}\label{eq17}
    \mathbf{z}_l^{\mathopen{}\left(k+1\right)\mathclose{}}= \argmin_{\mathbf{z}_l} \sum_{j=1}^n \frac{1}{2}f\mathopen{}\left(z_{l,j},\mu^{\mathopen{}\left(k+1\right)\mathclose{}}\right)\mathclose{} + \mathopen{}\left( \tau-\frac{1}{2}\right)\mathclose{}\mathbf{1}_{M_l}^{\text{T}} \mathbf{z}_l 
     \nonumber\\ 
   +\mathopen{}\left(\boldsymbol{\Psi}_l^{\mathopen{}\left(k\right)\mathclose{}}\right)\mathclose{}^\text{T} \mathbf{z}_l+   \frac{\sigma_{\Psi}^{\mathopen{}\left(k+1\right)\mathclose{}}}{2}\mathopen{}\left\|\mathbf{z}_l+\mathbf{X}^{(l)}\mathbf{w}_l^{\mathopen{}\left(k+1\right)\mathclose{}}- \mathbf{y}^{(l)}\right\|_2^2\mathclose{}
\end{align}
It can be shown that the update step of each $\mathbf{z}_l$ in ADMM has a closed-form solution. By grouping the last three terms of \eqref{eq17} together, a component-wise solution can be obtained as
\begin{align}\label{eq:up:z}
\mathbf{z}_i^{\mathopen{}\left(k+1\right)\mathclose{}}=\textbf{Prox}_{f\mathopen{}\left(\cdot,\mu^{\mathopen{}\left(k+1\right)\mathclose{}}\right)\mathclose{}}\mathopen{}\left(\alpha_{l,i};\frac{\sigma_{\Psi}^{\mathopen{}\left(k+1\right)\mathclose{}}}{2}\right)_{i=1}^{M_l}\mathclose{}
\end{align}
where
$\boldsymbol{\alpha}_l=\mathopen{}\left(\mathbf{y}^{(l)}-\mathbf{X}^{(l)}\mathbf{w}_l^{\mathopen{}\left(k+1\right)\mathclose{}}\right)\mathclose{}-\frac{\boldsymbol{\Psi}_l^{\mathopen{}\left(k\right)\mathclose{}}+\mathopen{}\left(\tau-\frac{1}{2}\right)\mathclose{}\mathbf{1}_{M_l}}{\sigma_{\Psi}^{\mathopen{}\left(k+1\right)\mathclose{}}}$, and \begin{equation}\label{eq19}
\textbf{Prox}_{f\mathopen{}\left(\cdot,{\mu}\right)\mathclose{}}\mathopen{}\left(x;\rho\right)\mathclose{}=
\begin{cases}
x-\rho, & x \geq  \rho + \mu
\\
\frac{z}{1+\frac{\rho}{\mu}}, & -\rho - \mu < x < \rho + \mu 
\\
x+\rho. &  x < -\rho - \mu 
\end{cases}
\end{equation}

Updates for both $\mathbf{g}_{lj}$ and $\mathbf{g}_{jl}$ for $j \in \mathcal{N}_l$, $j < l$ can be performed in parallel as:
\begin{multline}
  \begin{bmatrix}
       \mathbf{g}_{lj}^{(k+1)}\\
    \mathbf{g}_{jl}^{(k+1)}
\end{bmatrix} =\argmin_{\mathbf{g}_{lj},\mathbf{g}_{jl}}  \omega \sum_{p=1}^{P+1} f({g}_{lj,p}-{g}_{jl,p},\mu^{\mathopen{}\left(k+1\right)\mathclose{}})+\frac{\sigma_{\xi}}{2}\times \\ 
\phantom{=}\Bigg(   \left\|\mathbf{g}_{lj}-\mathbf{w}_{l}^{(k+1)}+\frac{\boldsymbol{\xi}_{jl}^{(k)}}{\sigma_{\xi}^{(k+1)}}\right\|^2_2+ \left\|\mathbf{g}_{jl}-\mathbf{w}_{l}^{(k+1)}+\frac{\boldsymbol{\xi}_{jl}^{(k)}}{\sigma_{\xi}}\right\|^2_2\Bigg)  
\end{multline}
Following the simplification provided in \cite{hallac2017network}, we get:
\begin{multline}\label{up:g}
     \begin{bmatrix}
    \mathbf{g}_{lj}^{(k+1)}\\
    \mathbf{g}_{jl}^{(k+1)}
\end{bmatrix}= \\
\phantom{=}\frac{1}{2}\begin{bmatrix}
  \mathbf{w}_{l}^{(k+1)}-\frac{\boldsymbol{\xi}_{lj}^{(k)}}{\sigma_{\xi}^{(k+1)}}+\mathbf{w}_{j}^{(k+1)}-\frac{\boldsymbol{\xi}_{jl}^{(k)}}{\sigma_{\xi}^{(k+1)}}\\
  \mathbf{w}_{i}^{(k+1)}-\frac{\boldsymbol{\xi}_{lj}^{(k)}}{\sigma_{\xi}^{(k+1)}}+\mathbf{w}_{j}^{(k+1)}-\frac{\boldsymbol{\xi}_{jl}^{(k)}}{\kappa_{\xi}^{(k+1)}}
\end{bmatrix}
+\frac{1}{2}\begin{bmatrix}
 -\mathbf{e}\\
\mathbf{e}
\end{bmatrix}
\end{multline}
where for each component \( p \in  \{1, \ldots, P+1\} \),
$
{e}_p = \textbf{Prox}_{f\left(\cdot, \mu\right)}\left(
{w}_{l,p}^{(k+1)} - \frac{{\xi}_{lj,p}^{(k)}}{\sigma_{\xi}^{(k+1)}}
- {w}_{j,p}^{(k+1)} + \frac{{\xi}_{jl,p}^{(k)}}{\sigma_{\xi}^{(k+1)}};
\frac{2\omega}{\kappa_{\xi}^{(k+1)}}
\right).
$
\begin{algorithm}[t]
 \caption{Decentralzied Smoothing ADMM (DSAD) for Penalized Quantile Regression}
 \label{alg:1}
\SetAlgoLined
Initialize $c$, $d$, $\beta$, $K$ for each node $i$, and the parameter $\tau$ and the regularized parameters $\gamma$ and $\lambda$.\;
 \For{$k=1,\cdots,K$}{
 Update $\mu^{(k+1)}$, $\sigma^{(k+1)}_{\Psi}$ and $\sigma^{(k+1)}_{\xi}$ by \eqref{eq:up:sm}\;
%  Each agent $l \in [1,\ldots,L]$ update its parameters as following:
\For{$l=1,\cdots,L$}{
  \For{$p=1,\cdots,P$}{
  Update $w^{(k+1)}_{l,p}$ by \eqref{up:wp}\;
 }
 Update $w^{(k+1)}_{l,P+1}$ by \eqref{eq:up:w:p+1} and send $\mathbf{w}^{(k+1)}_l$ to $\mathcal{N}_l$\;
Receive $\mathbf{w}^{(k+1)}_l$ by each node $j\in \mathcal{N}_l$\;
 Update $\mathbf{z}^{(k+1)}_l$ by \eqref{eq:up:z} and $\boldsymbol{\Psi}^{(k+1)}_l$ by \eqref{eq:up:ps}\;
 }
 \For{$l=1,\cdots,L$}{
  \For{$j \in \mathcal{N}_l, j>l$}{
  Update $\mathbf{g}^{(k+1)}_{lj}$ and $\mathbf{g}^{(k+1)}_{jl}$ by \eqref{up:g}\;
  Update $\boldsymbol{\xi}^{(k+1)}_{lj}$ and $\boldsymbol{\xi}^{(k+1)}_{jl}$ by \eqref{eq:up:gl} and \eqref{eq:up:gj}\;
 }
 }
 }
\end{algorithm}

Finally, the update of each dual variable $\boldsymbol{\Psi}_l$, $\boldsymbol{\xi}_{lj}$ and $\boldsymbol{\xi}_{jl}$ as $j\in \mathcal{N}_l, l<j$  is given by
 \begin{equation}
     \label{eq:up:ps}
\boldsymbol{\Psi}_l^{\mathopen{}\left(k+1\right)\mathclose{}}=\boldsymbol{\Psi}_l^{\mathopen{}\left(k\right)\mathclose{}}+\sigma_{\Psi}^{\mathopen{}\left(k+1\right)\mathclose{}}\mathopen{}\left(\mathbf{z}_l^{\mathopen{}\left(k+1\right)\mathclose{}}+ \mathbf{X}^{(l)}\mathbf{w}_l^{\mathopen{}\left(k+1\right)\mathclose{}}- \mathbf{y}^{(l)}\right)\mathclose{}
\end{equation}
\begin{equation}\label{eq:up:gl}
\boldsymbol{\xi}^{(k+1)}_{lj}=\boldsymbol{\xi}^{(k)}_{lj}+\sigma^{(k+1)}_{\xi}\left(\mathbf{w}^{(k+1)}_l-\mathbf{g}^{(k+1)}_{lj}\right)
\end{equation} 
\begin{equation} \label{eq:up:gj}
\boldsymbol{\xi}^{(k+1)}_{jl}=\boldsymbol{\xi}^{(k)}_{jl}+\sigma_{\xi}^{(k+1)}\left(\mathbf{w}^{(k+1)}_j-\mathbf{g}^{(k+1)}_{jl}\right)
 \end{equation}
 
 The proposed ADMM-based method for solving the decentralized sparse-penalized quantile regression is summarized in Algorithm \ref{alg:1} and the convergence is established by the following theorem. % \ref{theorem1}. 
 \begin{figure*}[t]
     \centering
     \begin{subfigure}[b]{0.3\textwidth}
         \centering
    \includegraphics[width=\textwidth]{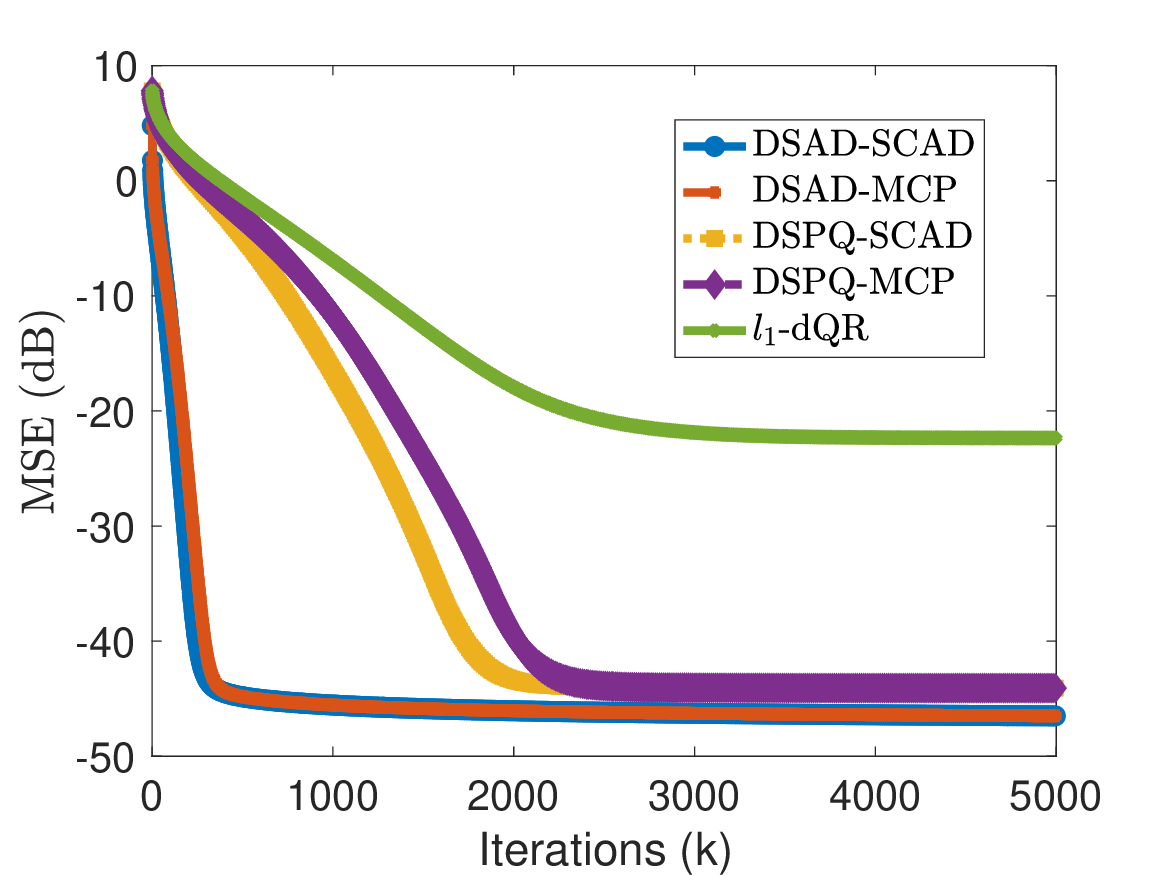}
    \caption{MSE versus iteration} \label{fig:t3}
    \label{fig1:1}
     \end{subfigure}
     \hfill
     \begin{subfigure}[b]{0.3\textwidth}
         \centering
    \includegraphics[width=\textwidth]{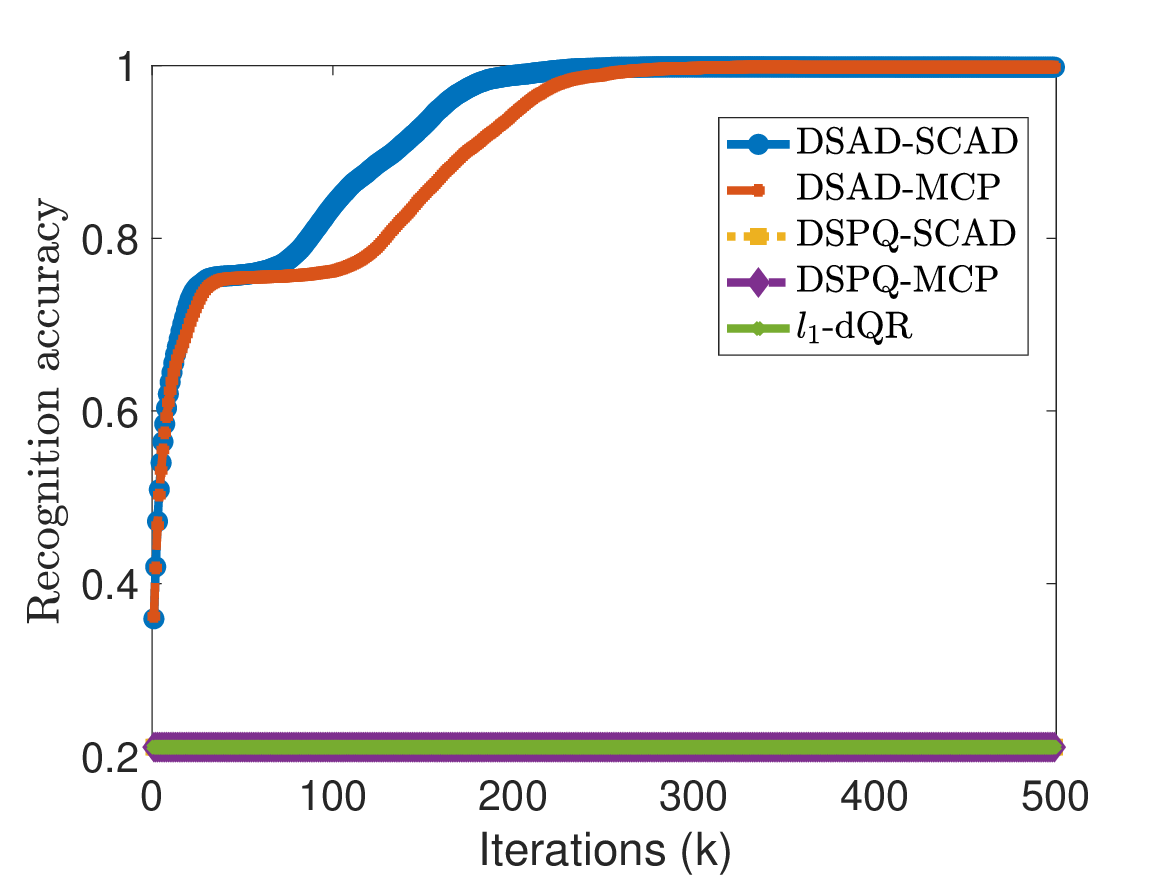}
    \caption{Recognition accuracy versus iteration} \label{fig:t5}
    \label{fig1:2}
     \end{subfigure}
     \hfill
     \begin{subfigure}[b]{0.3\textwidth}
         \centering
    \includegraphics[width=\textwidth]{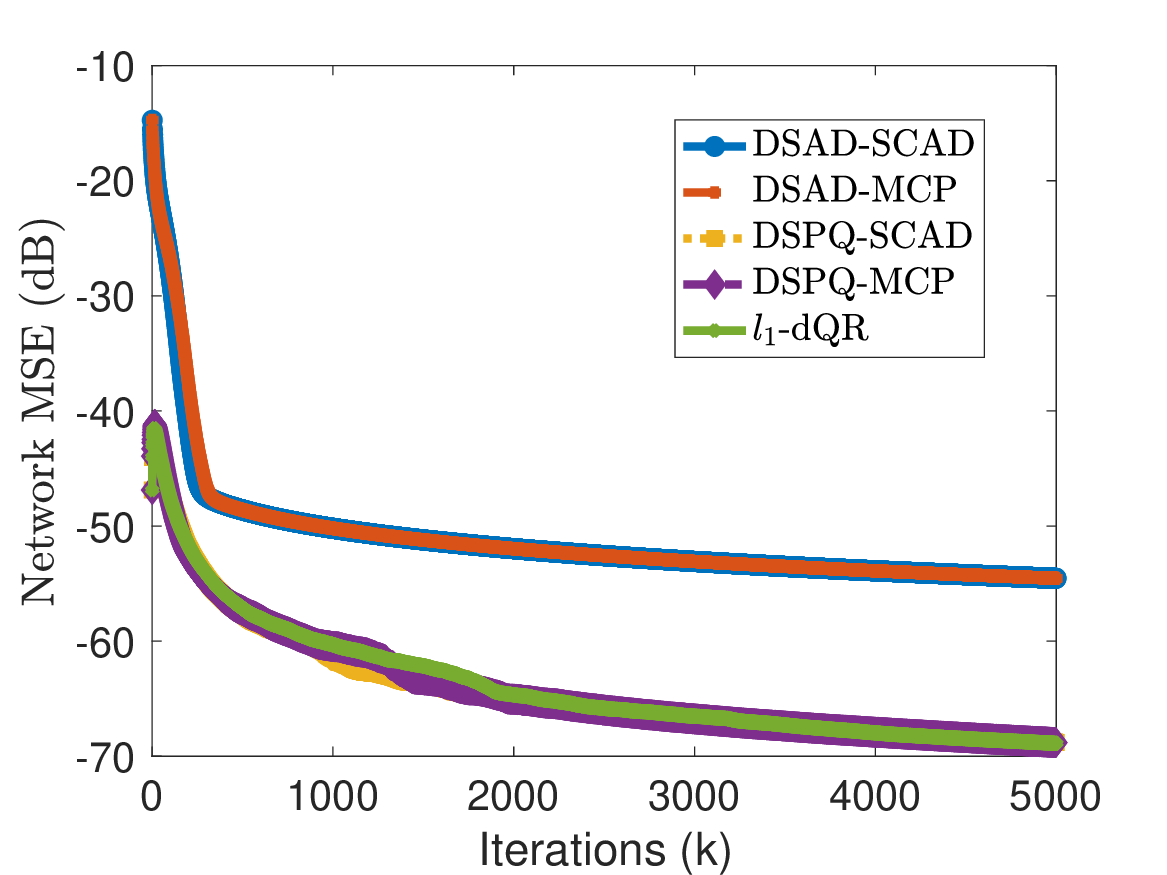}
    \caption{Network MSE versus iteration} \label{fig:t7}
    \label{fig1:3}
     \end{subfigure}
     \centering
   \caption{Performance comparison of DSAD, DSPQ, and $l_1$-dQR}
   \label{fig1}
\end{figure*}
 \begin{theorem}[Global Convergence]\label{theorem1}
Let $K$ be a constant such that for every $k \geq K$, we have  $\sigma^{(k+1)}_{\xi}>\frac{M \rho}{\min_{p,l} \|\mathbf{X}^{(l)}_{:,p}\|^2_2}$,  where %$\rho$ is 
$\rho \geq \frac{1}{\gamma}$ for MCP and $\rho \geq \frac{1}{\gamma-1}$ for SCAD represent their weakly convexity conditions. Assume each node has $M$ samples and that $\mu^{(k+1)}$, $\sigma^{(k+1)}_{\Psi}$ and $\sigma^{(k+1)}_{\xi}$ are updated according to \eqref{eq:up:sm}, where $\beta c\geq \sqrt{\frac{3}{2}}$, $\beta d \geq \sqrt{20} \hspace{1 mm} \omega$, and $\omega > \max\{\tau,1-\tau\} \max_{l\in\{1,\cdots,L\}}\|(\mathbf{X}^{(l)})^{\textbf{T}}\|_{\infty}+M\lambda$. Then, Algorithm \ref{alg:1} converges to a stationary point $\mathopen{}\left(\mathbf{w}^{*},\mathbf{z}^{*},\boldsymbol{\Psi}^{*}\right)\mathclose{}$ that satisfies the following KKT conditions:
\begin{subequations}
\begin{align}
&\mathbf{X}^{\textbf{T}}\boldsymbol{\Psi}^{*}\in n \partial P_{\lambda,\gamma}\mathopen{}\left(\mathbf{w}^{*}\right)\mathclose{} \\
& \boldsymbol{\Psi}^{*} \in \partial \sum_{i=1}^n \rho_{\tau}\mathopen{}\left(z_i^{*}\right)\mathclose{} \\
& \mathbf{z}^{*}+\mathbf{X}\mathbf{w}^{*}-\mathbf{y}=0
\end{align}
\end{subequations}
\end{theorem}
\begin{proof}
We establish the convergence of the proposed ADMM algorithm by validating the following two steps: 

\noindent{\em Step 1. Convergence of variables:}
  Consider a \( K' \) greater than \( K \). The overall change in the augmented Lagrangian from iteration \( K \) to iteration \( K' \) is given by:
   \begin{multline}
{\bar{\mathcal{L}}}_{\sigma_{\Psi}^{(K')},\sigma_{\xi}^{(K')},\mu^{(K')}}\mathopen{}\left(\mathbf{W}^{(K')},\mathbf{Z}^{(K')},\mathbf{G}^{(K')},\boldsymbol{\Psi}^{(K')},\boldsymbol{\xi}^{(K')}\right)\mathclose{}-\\ 
{\bar{\mathcal{L}}}_{\sigma_{\Psi}^{(K)},\sigma_{\xi}^{(K)},\mu^{(K)}}\mathopen{}\left(\mathbf{W}^{(K)},\mathbf{Z}^{(K)},\mathbf{G}^{(K)},\boldsymbol{\Psi}^{(K)},\boldsymbol{\xi}^{(K)}\right)\mathclose{}\leq \\ -S_{K'}-S'_{K'}+ D
   \end{multline}
   where \( D \) is a constant due to the effect of \( \mu \) change on the augmented Lagrangian, which is summable as outlined in \cite{mirzaeifard2023smoothing}[Lemma 9]. The term \( S_{K'}\) is defined as:
   \begin{multline}
S_{K'} = \sum_{k=K}^{K'-1} \sum_{l=1}^L \eta^{(k+1)}\|\mathbf{z}_l^{(k+1)} - \mathbf{z}_l^{(k)}\|_2^2 \\+ \varkappa^{(k+1)}\|\mathbf{w}_l^{(k+1)}-\mathbf{w}_l^{(k)}\|_2^2
\end{multline}
where \( \eta^{(k+1)} = \left(\frac{\sigma_{\Psi}^{(k+1)}}{2} - \frac{1}{2(\mu^{(k+1)})^2\sigma_{\Psi}^{(k+1)}} - \frac{\sigma_{\Psi}^{(k+1)}-\sigma_{\Psi}^{(k)}}{2(\sigma_{\Psi}^{(k)})^2(\mu^{(k)})^2}\right) \) and \( \varkappa^{(k)}= \frac{\sigma_{\Psi}^{\mathopen{}\left(k\right)\mathclose{}} \min_p \|\mathbf{X}{:,p}\|^2_2}{2M} - \frac{\rho}{2} \). Additionally, \( S_{K'} \) is characterized by:
\begin{multline}
S_{K'} = \sum_{k=K}^{K'-1} \sum_{l=1}^L \sum_{j\in \mathcal{N}_l,j>l}  \iota^{(k+1)} \|\mathbf{q}_{lj}^{(k+1)}-\mathbf{q}_{lj}^{(k)}\|^2_2
\end{multline}
where \( \iota^{(k+1)}=\left(\frac{\sigma_{\xi}^{(k+1)}}{4} - \frac{4\omega^2}{(\mu^{(k+1)})^2\sigma_{\xi}^{(k+1)}} - \frac{2\omega^2(\sigma_{\xi}^{(k+1)}-\sigma_{\xi}^{(k)})}{(\sigma_{\xi}^{(k)})^2(\mu^{(k)})^2}\right) \), and \( \mathbf{q}_{lj}^{(k+1)}=\mathbf{g}_{lj}^{(k+1)}-\mathbf{g}_{jl}^{(k+1)} \). The conditions for \( \sigma^{(k+1)}_{\Psi} \), \( \sigma^{(k+1)}_{\xi} \), and \( \mu^{(k+1)} \) ensure that each of \( \iota^{(k+1)} \), \( \eta^{(k+1)} \), and \( \varkappa^{(k+1)} \) is positive. Given that the quantile regression function and \( l_1 \)-norm are positive functions with bounded gradients, it can be concluded that the augmented Lagrangian is lower bounded. Consequently, as $K'\rightarrow \infty$, the norm of successive differences for each variable converges to zero at a rate of \( o\left(k^{-\frac{3}{4}}\right)\), ensuring stabilization at limit points.

\noindent{\em Step 2. Consensus among nodes and centralized local minimum:} Following \cite{mirzaeifard2023smoothing}[Theorem 2], the derivative of the augmented Lagrangian converges toward zero at the rate of $o\left(k^{-\frac{1}{4}}\right)$, as the convergence of variables is on the order of $o\left(k^{-\frac{3}{4}}\right)$ and the growth rate of penalty parameters is on the order of $\Omega\left(k^{\frac{1}{2}}\right)$. Moreover, as the approximation functions converge to the actual functions, we can conclude that the limit points satisfy the KKT conditions. One of the KKT conditions in each node concerns local optimality: $\sum_{j\in \mathcal{N}_l,j<l} \boldsymbol{\xi}^{*}_{jl} + \sum_{j\in \mathcal{N}_l,j>l} \boldsymbol{\xi}^{*}_{lj} \in M \partial P_{\lambda,\gamma}\left(\mathbf{w}_l^{*}\right) - (\mathbf{X}^{(l)})^{\textbf{T}}\boldsymbol{\Psi}_l^{*}$. Moreover, considering the edge equality constraints $\mathbf{g}_{lj}=\mathbf{w}_l, \quad \mathbf{g}_{jl}=\mathbf{w}_j, \forall j \in \mathcal{N}_l, j>l$ implies $\boldsymbol{\xi}^{*}_{lj} \in -\omega \partial_{\mathbf{w}_l} {\|\mathbf{w}_l^{*}-\mathbf{w}_j^{*}\|_1}$. Thus, there exists $\mathbf{d}_l \in -\sum_{j \in \mathcal{N}_l} \omega \partial_{\mathbf{w}_l} \|\mathbf{w}_l^{*}-\mathbf{w}_j^{*}\|_1$, such that $\mathbf{d}_l \in M \partial P_{\lambda,\gamma}\left(\mathbf{w}_l^{*}\right) - (\mathbf{X}^{(l)})^{\textbf{T}}\boldsymbol{\Psi}_l^{*}$. Given the condition for $\omega$ and since each $\partial_{\mathbf{w}_l} \|\mathbf{w}_l^{*}-\mathbf{w}_j^{*}\|_1=-\partial_{\mathbf{w}_j} \|\mathbf{w}_l^{*}-\mathbf{w}_j^{*}\|_1$, the consensus among all $\mathbf{w}_l^{*}$ is inevitable to satisfy the local optimally in the KKT conditions. Finally, consolidating all parts of the KKT conditions into a centralized form, the limit points also satisfy the centralized KKT conditions.
\end{proof}
\section{Simulation Results}
This section evaluates the DSAD algorithm through simulations against DSPQ \cite{mirzaeifard2023distributed} and $l_1$-dQR \cite{wang2017distributed}. We model a sensor network comprising $L=30$ nodes distributed over a $2.5 \times 2.5$ square area, with connections between nodes located within $0.8$ units of each other. Each node has between $2$ and $10$ neighbors. Each node $l$ contains $M_l=500$ data samples, with each measurement $\mathbf{x}^{(l)}_{j}$ drawn from an i.i.d normal distribution $\mathcal{N} (0,\boldsymbol{\Sigma}_{P \times P})$ where $P=18$ and $\Sigma_{pq}=0.5^{|p-q|}$. The response variable $y^{(l)}_{j}$ is modeled as $y^{(l)}_{j} = (\mathbf{x}^{(l)}_{j})^{\text{T}} \boldsymbol{\beta}+\epsilon^{(l)}_{j}$, where $\boldsymbol{\beta}$ has $3$ randomly selected active coefficients set to one, and $\epsilon^{(l)}_{j}$ is noise generated from $\mathcal{N}(0,0.2)$. The quantile of the noise is denoted by $q_\tau^\epsilon = 0.2 \cdot \Phi(\tau)$, with $\Phi(\cdot)$ as the cumulative distribution function of the standard normal distribution. Parameters utilized include $\gamma_{\text{SCAD}}=3.7$, $\gamma_{\text{MCP}}=2.4$, $\lambda=0.055$, $\tau=0.75$, $\omega = \max\{\tau,1-\tau\} \|\mathbf{X}^{\textbf{T}}\|_{\infty}+M\lambda+1$, $d=\sqrt{20}\omega$, $c=\sqrt{\frac{3}{2}}$, and $\beta=1$.

Performance metrics include mean square error (MSE), calculated as $\frac{\sum_{l=1}^L||\mathbf{\hat{w}}_l-\mathbf{w}||_2^2}{L}$, network MSE as $\frac{\sum_{l=1}^L||\mathbf{\hat{w}}_l-\bar{\mathbf{w}}||_2^2}{L}$, and recognition accuracy, defined as the proportion of correctly identified active and non-active coefficients relative to the total number of coefficients. These metrics were averaged over $100$ trials. As depicted in Fig. \ref{fig1:1}, DSAD demonstrates enhanced MSE performance with both MCP and SCAD penalties compared to DSPQ and $l_1$-dQR. Furthermore, as shown in Fig. \ref{fig1:2}, DSAD excels in identifying active and non-active coefficients, a task where other algorithms falter as they are not able to provide exact zero values for non-active coefficients. However, Fig. \ref{fig1:3} illustrates that DSPQ and $l_1$-dQR achieve faster consensus among nodes than DSAD.

\section{Conclusion}
This paper introduced an ADMM algorithm for decentralized quantile regression utilizing non-convex and non-smooth sparse penalties. The convergence of the proposed algorithm was rigorously verified. Our simulation outcomes underscore the algorithm's superiority in terms of mean squared error and recognition accuracy relative to existing methods like DSPQ and $l_1$-dQR. Notably, the algorithm demonstrates exceptional effectiveness in identifying active and non-active coefficients, significantly outperforming other algorithms in scenarios requiring precise coefficient discrimination.
\newpage

\bibliographystyle{IEEEtran}

\end{document}